\documentclass[11pt]{article}

\usepackage{amsmath}
\usepackage{amssymb}
\usepackage{amsthm}
\usepackage{enumitem}
\usepackage{graphicx}
\usepackage[margin=1.25in]{geometry}
\usepackage[round,authoryear,comma]{natbib}
\usepackage{times}
\usepackage{url}
\usepackage{wrapfig}
\usepackage{xcolor}

\newcommand{\HmyT}{{\hat{\hhc{C}}_{f,n}}}
\newcommand{\M}{\mathcal{M}} 
\newcommand{\R}{\mathbb R}
\newcommand{\T}{{\hhc{C}_f}}
\newcommand{\X}{\mathcal X}

\newcommand{\eps}{\varepsilon}
\newcommand{\fe}{\mathcal{E}}
\newcommand{\hatT}{{\hhc{\hat C}_{f,n}}}
\newcommand{\hc}{\mathcal}
\newcommand{\hhc}{\mathsf}
\newcommand{\knminus}{\frac{k}{n} - \frac{C_\delta}{n}\sqrt{k d \log n}}
\newcommand{\knplus}{\frac{k}{n} + \frac{C_\delta}{n}\sqrt{k d \log n}}

\newcommand{\myT}{T^\myr_n} 
\newcommand{\mygamma}{\Gamma}
\newcommand{\mylip}{\mathsf{c}}
\newcommand{\myr}{{\mathsf{r}}}

\newcommand{\optT}{{\hhc{C}_f}}
\newcommand{\splittree}{split-cluster tree}
\newcommand{\suplev}[1][\lambda]{\{x \in \X : f(x) \geq #1\}}

\newtheorem{definition}{Definition}
\newtheorem{lemma}{Lemma}
\newtheorem{theorem}{Theorem}

\title{
    Beyond Hartigan Consistency:\\ 
    Merge Distortion Metric for Hierarchical Clustering
    }

\author{
    Justin Eldridge\\
    \small\texttt{eldridge@cse.ohio-state.edu}
    \and
    Mikhail Belkin\\
    \small\texttt{mbelkin@cse.ohio-state.edu}
    \and
    Yusu Wang\\
    \small\texttt{yusu@cse.ohio-state.edu}
    }

\date{}

\begin{document}

\maketitle

\begin{abstract} 

Hierarchical clustering is a popular method for analyzing data which associates
a tree to a dataset. Hartigan consistency has been used extensively as a
framework to analyze such clustering algorithms from a statistical point of
view. Still, as we show in the paper, a tree which is Hartigan consistent with a
given density can look very different than the correct limit tree. Specifically,
Hartigan consistency permits two types of undesirable configurations which we
term \emph{over-segmentation} and \emph{improper nesting}.  Moreover, Hartigan
consistency is a limit property and does not directly quantify difference
between trees.

In this paper we identify two limit properties, \emph{separation} and
\emph{minimality}, which address both over-segmentation and improper nesting and
together imply (but are not implied by) Hartigan consistency. We proceed to
introduce a \emph{merge distortion metric} between hierarchical clusterings and
show that convergence in our distance implies both separation and minimality. We
also prove that uniform separation and minimality imply convergence in the merge
distortion metric.  Furthermore, we show that our merge distortion metric is
stable under perturbations of the density.

Finally, we demonstrate applicability of these concepts by proving convergence
results for two clustering algorithms.   First, we show convergence (and hence
separation and minimality) of the recent robust single linkage algorithm of
\cite{chaudhuri_2010}. Second, we provide convergence results on
manifolds for  topological  split tree clustering.
\end{abstract} 

\section{Introduction}

Hierarchical clustering is an important class of techniques and algorithms for
representing data in terms of a certain tree structure~\citep{jain_1988}. When
data are sampled from a probability distribution, one needs to study the
relationship between trees obtained from data samples to the infinite tree of
the underlying probability density. This question was first explored in
\citet{hartigan_1975}, which introduced the notion of \emph{high-density
clusters}. Specifically, given density function $f : \X \to \R$, the
high-density clusters are defined to be the connected components of $\suplev$
for some $\lambda$. The set of all clusters forms a hierarchical structure known
as the \emph{density cluster tree} of $f$. The natural notion of consistency for
finite density estimators is is to require that any two high density clusters
are also separate in the finite tree given enough samples.  This notion was
introduced in \citet{hartigan_1981} and is known as {\it Hartigan consistency}.
Still, while clearly desirable, it is well known that Hartigan consistency does
not fully capture the properties of convergence that one would \emph{a priori}
expect. In particular, it does not exclude trees which are very different from
the underlying probability distribution. 

In this paper we identify two distinct undesirable configuration types permitted
by Hartigan consistency, \emph{over-segmentation} \cite[identified as the
problem of \emph{false clusters} in][]{chaudhuri_2014} and \emph{improper
nesting}, and show how both of these result from clusters merging at the wrong
level. To address these issues we propose two basic properties for hierarchical
cluster convergence: \emph{minimality} and \emph{separation}. Together they
imply Hartigan consistency and, furthermore, rule out ``improper"
configurations.  We proceed to introduce a \emph{merge distortion} metric on
clustering trees and show that convergence in the metric implies both separation
and minimality. Moreover, we demonstrate that uniform versions of these
properties are in fact equivalent to metric convergence.  We note that the
introduction of a quantifiable \emph{merge distortion} metric also addresses
another issue with Hartigan consistency, which is a limit property of clustering
algorithms and is not  quantifiable as such. We also prove that the merge
distortion metric is robust to small perturbations of the density.

Still,  attempts to formulate some intuitively desirable properties of
clustering have led to well-known impossibility results, such as those proven
by~\citet{kleinberg_2003}.  In order to show that our definitions correspond to
actual objects, and, furthermore, to realistic algorithms, we analyze the
robust single linkage clustering proposed by~\citet{chaudhuri_2010}.  We prove
convergence of that algorithm under our merge distortion metric and hence show
that it satisfies separation and minimality conditions.  We also  propose  a
topological split tree algorithm for hierarchical clustering (based on the
algorithm introduced by~\cite{CGOS13} for flat clustering) and demonstrate  its
convergence on Riemannian manifolds.

\paragraph{Previous work.} The problem of devising an algorithm which provably
converges to the true density cluster tree in the sense of Hartigan has a long
history.  \cite{hartigan_1981} proved that single linkage clustering is
\emph{not} consistent in dimensions larger than one. Previous to this,
\citet{wishart_1969} had introduced a more robust version of single linkage, but
its consistency had not been known. \citet{stuetzle_2010} introduced another
generalization of single-linkage designed to estimate the density cluster tree,
but again consistency was not established. Recently, however, two distinct
consistent algorithms have been introduced: The robust single linkage algorithm
of \citet{chaudhuri_2010}, and the tree pruning method of \cite{kpotufe_2011}.
Both algorithms are analyzed together, along with a pruning extension, in
\citet{chaudhuri_2014}. The robust single linkage algorithm was generalized in
\citet{balakrishnan_2013} to densities supported on a Riemannian submanifold of
$\R^d$.  We analyze the algorithm of \citet{chaudhuri_2010} in
Section~\ref{section:chaudhuri}. \cite{chaudhuri_2010} provide several theorems
which make precise the sense in which clusters are connected and separated at
each step of the robust single linkage algorithm. This paper translates their
results to our formalism, thereby proving that robust single linkage converges
to the density cluster tree in the merge distortion metric.

A central contribution of this paper will be to introduce notions which extend
Hartigan consistency, and are desirable properties of any algorithm which
estimates the density cluster tree.  In a related direction,
\citet{kleinberg_2003} outlined three desirable properties of a clustering
method, and proved that no method satisfying all three exists.
\cite{Ben-David_Ackerman_2009} argued that the impossibility result of Kleinberg
is tied to his formalism, and showed that axioms similar to his can be
made consistent by axiomatizing clustering quality measures as opposed to
clustering functions themselves.  \cite{Zadeh_Ben-David_2009} and
\cite{Ackerman_Ben-David_Loker_2010} presented axiomatic characterizations of
linkage-based clustering algorithms. Similarly, \cite{facundo_2010} introduced
\emph{functoriality} as one of three axioms related to Kleinberg's and showed
that single linkage agglomerative clustering is the only method which
simultaneously satisfies each.


\section{Preliminaries and definitions} 

A \emph{clustering} $\hc C$ of a set $X$ is the organization of its elements
into a collection of subsets of $X$ called \emph{clusters}. In general, clusters
may overlap or be disjoint. If the collection of clusters exhibits nesting
behavior (to be made precise below), the clustering is called
\emph{hierarchical}.  The nesting property permits us to think of a hierarchical
clustering as a tree of clusters, henceforth called a \emph{cluster tree}.  
\begin{definition}[Cluster tree]
\label{def:cluster_tree}
A \emph{cluster tree} (hierarchical clustering) of a set $X$ is a collection
$\hc{C}$ of subsets of $X$ s.t. $X \in \hc{C}$ and
 $\hc{C}$ has hierarchical structure. 
That is, if $C,C' \in \hc{C}$
        such that $C \neq C'$, then $C \cap C' = \emptyset$, or $C \subset C'$
        or $C' \subset C$.
Each element $C$ of $\hc{C}$ is called a \emph{cluster}. Each cluster $C$ is a
node in the tree. The \emph{descendants} of $C$ are those clusters $C' \in
\hc{C}$ such that $C' \subset C$. Every cluster in the tree except for $X$
itself is a descendant of $X$, hence $X$ is called the \emph{root} of the
cluster tree.
\end{definition}
Note that our definition of a cluster tree does not assume that either the set
of objects $X$ or the collection of clusters $\hc{C}$ are finite or even
countable. Hierarchical clustering is commonly formulated as a sequence of
nested partitions of $X$ \citep[see]{jain_1988}, culminating in the partition of
$X$ into singleton clusters. Our formulation differs in that it is a sequence of
nested partitions of \emph{subsets} of $X$. Notably, we don't impose the
requirement that $\{x\}$ appear as a cluster for every $x$.

Given a density $f$ supported on $\X \subset \R^d$, a natural way to cluster
$\X$ is into regions of high density. \citet{hartigan_1975} made this notion
precise by defining a \emph{high-density cluster} of $f$ to be a connected
component of the superlevel set $\{f \geq \lambda\} := \suplev$ for any $\lambda
\geq 0$. It is clear that this clustering exhibits the nesting property: If $C$
is a connected component of $\{ f \geq \lambda \}$, and $C'$ is a connected
component of $\{ f \geq \lambda' \}$, then either $C \subseteq C'$, $C'
\subseteq C$, or $C \cap C' = \emptyset$. We can therefore interpret the set of
all high-density clusters of a density $f$ as a cluster tree:

\begin{definition}[Density cluster tree of $f$]
\label{def:high_density_cluster_tree}
Let $\X \subset \R^d$ and consider any $f : \X \to \R$.  The \emph{density
cluster tree} of $f$, written $\hc{C}_f$, is the cluster tree whose nodes
(clusters) are the connected components of $\suplev$ for some $\lambda \geq 0$.
\end{definition}
We note that the density cluster tree of $f$ is closely related to the so-called
\emph{split tree} studied in the  computational geometry and topology literature
as a variant of the \emph{contour tree}; see e.g, \citep{CSA03}. 
We discuss a split tree-based approach to estimating the density cluster tree
in Section~\ref{section:split_tree}.

In practice we do not have access to the true density $f$, but rather a finite
collection of samples $X_n \subset \X$ drawn from $f$. We may attempt to recover
the structure of the density cluster tree $\hc{C}_f$ by applying a hierarchical
clustering algorithm to the sample, producing a discrete cluster tree $\hc{\hat
C}_{f,n}$ whose clusters are subsets of $X_n$.  In order to discuss the sense in
which the discrete estimate $\hc{\hat C}_{f,n}$ is consistent with the density
cluster tree $\hc{C}_f$ in the limit $n \to \infty$, \citet{hartigan_1981}
introduced a notion of convergence which has since been referred to as
\emph{Hartigan consistency}.  We follow \citet{chaudhuri_2010} in defining
Hartigan consistency in terms of the density cluster tree:
\begin{definition}[Hartigan consistency]
\label{definition:hartigan}
Suppose a sample $X_n \subset \X$ of size $n$ is used to construct a cluster
tree $\hc{\hat C}_{f,n}$ that is an estimate of $\hc{C}_f$.  For any sets $A,A'
\subset \X$, let $A_n$ (respectively $A'_n$) denote the smallest cluster of
$\hc{\hat C}_{f,n}$ containing $A \cap X_n$ (respectively, $A' \cap X_n$). We
say $\hc{\hat C}_{f,n}$ is consistent if, whenever $A$ and $A'$ are different
connected components of $\suplev$ for some $\lambda > 0$, $\operatorname{Pr}(A_n
\text{ is disjoint from } A'_n) \to 1$ as $n \to \infty$.
\end{definition}

In what follows, it will be useful to talk about the ``height'' at which two
points in a clustering merge. To motivate our definition, consider the two
points $a$ and $a'$ which sit on the surface of the density depicted in
Figure~\ref{fig:merge_height}. Intuitively, $a$ sits at height $f(a)$ on the
surface, while $a'$ sits at $f(a')$. If we look at the superlevel set
$\{ f \geq f(a) \}$, we see that $a$ and $a'$ lie in two different high-density
clusters. As we sweep $\lambda < f(a)$, the disjoint components of $\{ f \geq
\lambda \}$ containing $a$ and $a'$ grow, until they merge at height $\mu$. We
therefore say that the \emph{merge height} of $a$ and $a'$ is $\mu$. 

\begin{figure}[t]
\centering
\includegraphics[scale=1.4]{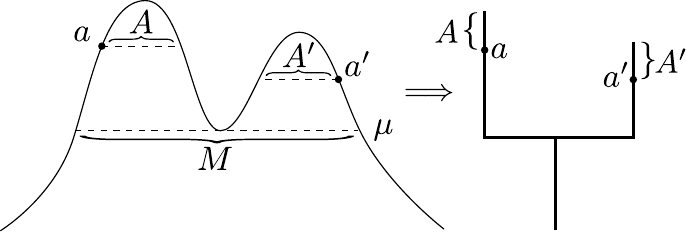}
\caption{\label{fig:merge_height} The density has a tree-like structure in which
$a$ and $a'$ merge at level $\mu$.}
\end{figure}

We may also interpret the situation depicted in Figure~\ref{fig:merge_height} in
the language of the density cluster tree. Let $A$ be the connected
component of $\{ f \geq f(a) \}$ which contains $a$, and let $A'$ be the
component of $\{ f \geq f(a') \}$ containing $a'$. Recognize that $A$ and $A'$
are nodes in the density cluster tree. As we walk the unique path from $A$
to the root, we eventually come across a node $M$ which contains both $a$ and
$a'$.  Note that $M$ is a connected component of the superlevel set $\{ f \geq
\mu \}$. It is desirable to assign a height to the entire cluster $M$, and a
natural choice is therefore $\mu$.

We extend this intuition to cluster trees which may not, in general, be
associated with a density $f$ by introducing the concept of a height function:

\begin{definition}[Cluster tree with height function]
A cluster tree with a height function is a triple $\hhc{C} = (X,
\hc{C}, h)$, where $X$ is a set of objects, $\hc{C}$ is a cluster tree of $X$,
and $h: X \to \R$ is a height function mapping each point in $X$ to a ``height".  
Furthermore, we define the height of a cluster $C \in \hc{C}$ to
be the lowest height of any point in the cluster. That is, $h(C) = \inf_{x \in
C} h(x)$.  Note that the nesting property of $\hc{C}$ implies that if $C'$ is a
descendant of $C$ in the cluster tree, then $h(C') \geq h(C)$.
\end{definition}

We will be consistent in using $\T$ to denote the density cluster tree of $f$
equipped with height function $f$. That is, $\T = (\X, \hc{C}_f, f)$.  Armed
with these definitions, we may precisely discuss the sense in which points --
and, by extension, clusters -- are connected at some level of a tree:

\begin{definition}
Let $\hhc{C} = (X, \hc{C}, h)$ be a hierarchical clustering of $X$ equipped with
height function $h$. 
\begin{enumerate}
\item Let $x,x' \in X$. We say that $x$ and $x'$ are \emph{connected at level
    $\lambda$} if there exists a $C \in \hc{C}$ with $x,x \in C$ such that $h(C)
    \geq \lambda$. Otherwise, $x$ and $x'$ are \emph{separated at level
    $\lambda$}. 
\item A subset $S\subset X$ is \emph{connected at level $\lambda$} if for any
    $s,s' \in S$, $s$ and $s'$ are connected at level $\lambda$.
\item Let $S \subset X$ and $S' \subset X$. We say that $S$ and $S'$ are
    \emph{separated at level $\lambda$} if for any $s \in S$, $s' \in S'$, $s$
    and $s'$ are separated at level $\lambda$.
\end{enumerate}
\end{definition}

We can now formalize the notion of \emph{merge height}:

\begin{definition}[Merge height]
\label{def:merge_height}
Let $\hhc{C} = (X, \hc{C}, h)$ be a hierarchical clustering equipped with a
height function. Let $x,x' \in X$, and suppose that $M$ is the smallest cluster
of $\hc{C}$ containing both $x$ and $x'$. That is, if $M' \in \hc{C}$ is a
proper sub-cluster of $M$, then $x \not \in M'$ or $x' \not \in M'$. We define
the \emph{merge height} of $x$ and $x'$ in $\hhc{C}$, written
$m_{\hhc{C}}(x,x')$, to be the height of the cluster $M$ in which the two points
merge, i.e., $m_{\hhc{C}}(x,x') = h(M)$. If $S \subset X$, we define the
\emph{merge height} of $S$ to be the $\inf_{(s,s') \in S \times S}
m_\hhc{C}(s,s')$.
\end{definition}

In what follows, we argue that a natural and advantageous definition of
convergence to the true density cluster tree is one which requires that,
for any two points $x,x'$, the merge height of $x$ and $x'$ in an estimate,
$m_{\hatT}(x,x')$, approaches the true merge height $m_{\T}(x,x')$ in the limit
$n \to \infty$.


\section{Notions of consistency}  

In this section we argue that while Hartigan consistency is a desirable
property, it is not sufficient to guarantee that an estimate captures the true
cluster tree in a sense that matches our intuition. We first illustrate the
issue by giving an example in which an algorithm is Hartigan consistent, yet
produces results which are very different from the true cluster tree.  We then
introduce a new, stronger notion of consistency which directly addresses the
weaknesses of Hartigan's definition.

\paragraph{The insufficiency of Hartigan consistency.}

An algorithm which is Hartigan consistent can nevertheless produce results
which are quite different than the true cluster tree.
Figure~\ref{fig:hartigan_problems} illustrates the issue.
Figure~\ref{fig:hartigan_problems}(a) depicts a two-peaked density $f$ from
which the finite sample $X_n$ is drawn. The two disjoint clusters $A$ and $B$
are also shown. The two trees to the right represent possible outputs of
clustering algorithms attempting to recover the hierarchical structure of $f$.
Figure~\ref{fig:hartigan_problems}(b) depicts what we would intuitively
consider to be an ideal clustering of $X_n$, whereas
Figure~\ref{fig:hartigan_problems}(c) shows an undesirable clustering which
does not match our intuition behind the density cluster tree of $f$.

First, note that while the two clusterings are very different, both satisfy
Hartigan consistency. Hartigan's notion requires only separation: The smallest
empirical cluster containing $A \cap X_n$ must be disjoint from the smallest
empirical cluster containing $B \cap X_n$ in the limit. The smallest empirical
cluster containing $A \cap X_n$ in the undesirable clustering is $A_n := \{x_2,
a_1, a_2, a_3\}$, whereas the smallest containing $B \cap X_n$ is $B_n :=
\{x_1, b_1, b_2, b_3\}$. $A_n$ and $B_n$ are clearly disjoint, and so Hartigan
consistency is not violated. In fact, the undesirable tree separates any pair
of disjoint clusters of $f$, and therefore represents a possible output of an
algorithm which is Hartigan consistent despite being quite different from
the true tree.

\begin{figure}
    \centering
    \includegraphics[scale=1]{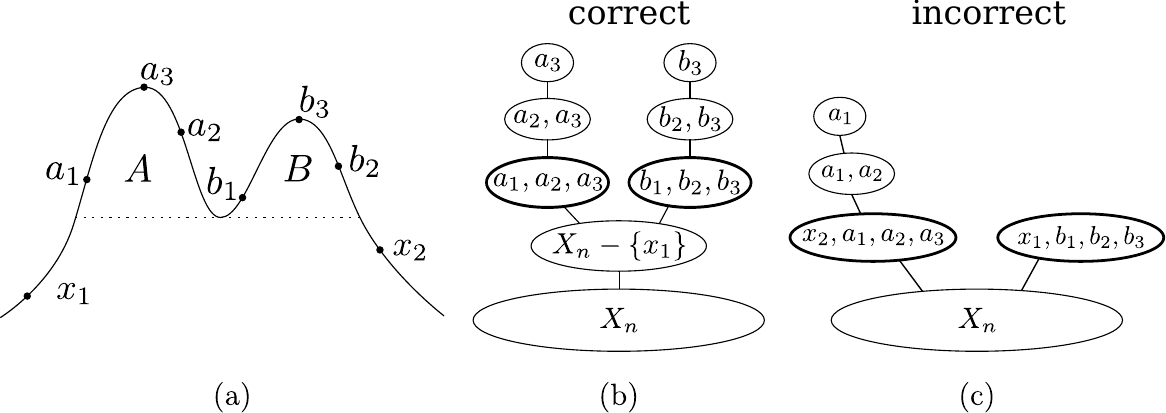}
    \caption{\label{fig:hartigan_problems} (c) is Hartigan consistent, yet looks
    rather different than the true tree.}
\end{figure}

We will show that the undesirable configurations of
Figure~\ref{fig:hartigan_problems}(c) arise because Hartigan consistency does
not place strong demands on the level at which a cluster should be connected.
Consider a cluster $A$ occurring at level $\lambda$ of the true density, and
let $A_n$ be the smallest empirical cluster containing all of $A \cap X_n$. In
the ideal case, an algorithm would perfectly recover $A$ such that $A_n = A
\cap X_n$. It is much more likely, however, that $A_n$ contains ``extra''
points from outside of $A$. Hartigan consistency places one constraint on the
nature of these extra points: They may not belong to some other disjoint
cluster of $f$. However, Hartigan's notion allows $A_n$ to contain points from
clusters which are \emph{not} disjoint from $A$. By their nature, these points
must be of density less than $\lambda$. If $A_n$ contains such extra points,
then $A \cap X_n$ is \emph{separated} at level $\lambda$, and in fact only
becomes connected at level $\min_{a \in A_n} f(a) < \delta$. Therefore,
permitting $A \cap X_n$ to become connected at a level lower than $\lambda$ is
equivalent to allowing ``extra'' points of density $< \lambda$ to be contained
within $A_n$.

The undesirable configurations depicted in
Figure~\ref{fig:hartigan_problems}(c) can be divided into two distinct
categories, which we term \emph{over-segmentation} and \emph{improper nesting}.
Either of these issues may exist independently of the other, and both are 
symptoms of allowing clusters to become connected at lower levels than what is
appropriate.

\emph{Over-segmentation} occurs when an algorithm fragments a true cluster,
returning empirical clusters which are disjoint at level $\lambda$ but are in
actuality part of the same connected component of $\{ f \geq \lambda \}$.  The
problem is recognized in the literature by \cite{chaudhuri_2014}, who refer to
it as the presence of \emph{false clusters}.
Figure~\ref{fig:hartigan_problems}(c) demonstrates over-segmentation by
including the clusters $A_n := \{x_2, a_1, a_2, a_3\}$ and $B_n := \{x_1, b_1,
b_2, b_3\}$. $A_n$ and $B_n$ are disjoint at level $f(x_1)$, though both are in
actuality contained within the same connected component of $\{f \geq f(x_1)\}$.

It is clear that over-segmentation is a direct result of clusters connecting at
the incorrect level. The severity of the issue is determined by the difference
between the levels at which the cluster connects in the density cluster tree
and the estimate. That is, if $A$ is connected at $\lambda$ in the density
cluster tree, but $A \cap X_n$ is only connected at $\lambda - \delta$ in the
empirical clustering, then the larger $\delta$ the greater the extent to which
$A$ is over-segmented.

\emph{Improper nesting} occurs when an empirical cluster $C_n$ is the smallest
cluster containing a point $x$, and $f(x) > \min_{c \in C_n} f(c)$. The
clustering in Figure~\ref{fig:hartigan_problems}(c) displays two instances of
improper nesting.  First, the left branch of the cluster tree has improperly
nested the cluster $\{a_1, a_2\}$, as it is the smallest cluster containing
$a_2$, yet $f(a_1) < f(a_2)$. The right branch of the same tree has also been
improperly nested in a decidedly ``lazier'' fashion: the cluster $\{x_1, b_1,
b_2, b_3\}$ is the smallest empirical cluster containing each of $b_1$, $b_2$,
and $b_3$, despite each being of density greater than $f(x_1)$.  Improper
nesting is considered undesirable because it breaks the intuition we have about
the containment of clusters in the density cluster tree; Namely, if $A \subset
A'$ and $a \in A$, $a' \in A'$, then $f(a) \geq f(a')$.

Note that like over-segmentation, improper nesting is caused by a cluster
becoming connected at a lower level than is appropriate. For instance, suppose
$C_n$ is improperly nested; That is, it is the smallest empirical cluster
containing some point $x$ such that $f(x) > \min_{c \in C_n} f(c)$. Let $\tilde
C$ be the connected component of $\{ f \geq f(x) \}$, and let $\tilde C_n$ be
the smallest empirical cluster containing all of $\tilde C \cap X_n$. Then $C_n
\subset \tilde C_n$ such that $\min_{c \in \tilde C_n} f(c) < f(x)$. In other
words, $\tilde C \cap X_n$ is connected only below $f(x)$.

As previously mentioned, it is not reasonable to demand that a cluster $A$
be perfectly recovered by a clustering algorithm. Rather, if $A$ is connected
at level $\lambda$ in the density cluster tree, we should allow $A \cap X_n$
to be first connected at a level $\lambda - \delta$ in the estimate, for some
small positive $\delta$. We make this notion precise with the following
definition:
\begin{definition}[$\delta$-minimal]
Let $A$ be a connected component of $\suplev$, and let $\hatT$ be an estimate
of the density cluster tree of $f$ computed from finite sample $X_n$.  $A$ is
\emph{$\delta$-minimal} in $\hatT$ if $A \cap X_n$ is connected at level
$\lambda - \delta$ in $\hatT$.
\end{definition}

Intuitively, each cluster of the density cluster tree should be
$\delta$-minimal in an empirical clustering for as small of a $\delta$ as
possible. For example, take any sample $x \in X_n$ and let $C$ be the connected
component of $\{ f \geq f(x) \}$ containing $x$.  Some examination shows that
$C$ is $0$-minimal in the ideal clustering depicted in
Figure~\ref{fig:hartigan_problems}(b). As the ideal clustering is free from
over-segmentation and improper nesting, it stands to reason that a cluster can
only exhibit these issues to the extent that it is $\delta$-minimal; The larger
$\delta$, the more severely a cluster may be over-segmented or improperly
nested.

\paragraph{Minimality and separation.}

We have identified two senses -- over-segmentation and improper nesting -- in
which a hierarchical clustering method can produce results which are
inconsistent with the density cluster tree, but which are not prevented by
Hartigan consistency. We have shown that both are symptoms of clusters becoming
connected at the improper level, and argued that the extent to which a cluster
is $\delta$-minimal controls the amount in which it is over-segmented or
improperly nested. With more and more samples, we'd like the extent to which a
clustering exhibits over-segmentation and improper nesting to shrink to zero.
We therefore introduce a notion of consistency which requires any cluster to be
$\delta$-minimal with $\delta \to 0$ as $n \to \infty$.

In the following, suppose a sample $X_n \subset \X$ of size $n$ is used to
construct a cluster tree $\hc{\hat C}_{f,n}$ that is an
estimate of $\hc{C}_{f,n}$, and let $\hatT$ be $\hc{\hat C}_{f,n}$ equipped with
$f$ as height function. Furthermore, it is assumed that each definition holds
with probability approaching one as $n \to \infty$.

\begin{definition}[Minimality]
\label{def:minimality}
We say that $\hatT$ ensures \emph{minimality} if given any connected component
$A$ of the superlevel set $\suplev$ for some $\lambda > 0$, $A \cap X_n$ is
connected at level $\lambda - \delta$ in $\hatT$ for any $\delta > 0$ as $n \to
\infty$.
\end{definition}

Minimality concerns the level at which a cluster is connected -- it says nothing
about the ability of an algorithm to distinguish pairs of disjoint clusters. For
this, we must complement minimality with an additional notion of consistency
which ensures separation. Hartigan consistency is sufficient, but does not
explicitly address the level at which two clusters are separated. We will
therefore introduce a slightly different notion, which we term
\emph{separation}:

\begin{definition}[Separation]
\label{def:separation}
We say that $\hatT$ ensures \emph{separation} if when $A$ and
$B$ are two disjoint connected components of $\{ f \geq \lambda \}$
merging at $\mu = m_{\T}(A \cup B)$, $A \cap X_n$ and $B \cap X_n$ are
separated at level $\mu + \delta$ in $\hatT$ for any $\delta > 0$ as $n \to
\infty$.
\end{definition}

It is interesting to note that Hartigan consistency contains some weak notion of
connectedness, as it requires the two sets $A \cap X_n$ and $B \cap X_n$ to be
connected into clusters $A_n$ and $B_n$ at the same level at which they are
separated. Our notion only requires that $A \cap X_n$ and $B \cap X_n$ be
disjoint at this level.  We ``factor out'' Hartigan consistency's idea of
connectedness, leaving separation, and replace it with a stronger notion of
minimality.

Taken together, minimality and separation imply Hartigan consistency.\\[-1.5em]

\begin{theorem}[Minimality and separation $\Longrightarrow$ Hartigan
consistency]
If a hierarchical clustering method ensures both separation and minimality,
then it is Hartigan consistent.
\end{theorem}
\begin{proof}
Let $A$ and $A'$ be disjoint connected components of the superlevel set
$\suplev$ merging at level $\mu$. Pick any $\lambda - \mu > \delta > 0$.
Definitions~\ref{def:minimality}~and~\ref{def:separation} imply that there
exists an $N$ such that for all $n \geq N$, $A \cap X_n$ and $A' \cap X_n$ are
separated and individually connected at level $\mu + \delta$. Assume $n \geq N$.
Let $A_n$ be the smallest cluster containing all of $A \cap X_n$, and $A'_n$ be
the smallest cluster containing all of $A' \cap X_n$. Suppose for a
contradiction that there is some $x \in X_n$ such that $x \in A_n \cap A_n'$.
Then either $A_n \subset A_n'$ or $A_n' \subset A_n$. In either case, there is
some cluster $C$ such that $h(C) \geq \mu + \delta$, $A_n \subset C$, and $A_n'
\subset C$. Since $A \cap X_n \subset A_n$ and $A' \cap X_n \subset A_n'$, this
contradicts the assumption that $A \cap X_n$ and $A' \cap X_n$ are separated at
level $\mu + \delta$. Hence $A_n \cap A_n' = \emptyset$.
\end{proof}

Minimality and separation have been defined as properties which are true for all
clusters in the limit. In addition, we may define stronger versions of these
concepts which require that all clusters approach minimality and separation
uniformly:

\begin{definition}[Uniform minimality and separation]
\label{def:uniform_connectedness_and_separation}
$\hatT$ ensures \emph{uniform minimality} if given any $\delta > 0$
there exists an $N$ depending only on $\delta$ such that for all $n \geq N$ and
all $\lambda$, any cluster $A \in \suplev$ is connected at level $\lambda -
\delta$. $\hatT$ is said to ensure \emph{uniform separation} if given any
$\delta > 0$ there exists an $N$ depending only on $\delta$ such that for all $n
\geq N$ and all $\mu$, any two disjoint connected components merging in
$\suplev[\mu]$ are separated at level $\mu + \delta$.
\end{definition}

The uniform versions of minimality and separation are equivalent to the weaker
versions under some assumptions on the density. The proof of the following
theorem is given in Appendix~\ref{apx:proof:nice_function_implies_uniform}.

\begin{theorem}
\label{thm:nice_function_implies_uniform}
If the density $f$ is bounded from above and is such that $\suplev$ contains
finitely many connected components for any $\lambda$, then any algorithm which
ensures minimality also ensures uniform minimality on $f$, and any algorithm
which ensures separation also ensures uniform separation.
\end{theorem}

In the next section, we will introduce a distance between hierarchical
clusterings, and show that convergence in this metric implies these
consistency properties.


\section{Merge distortion metric} 

The previous section introduced the notions of  minimality and
separation, which are desirable properties for a hierarchical clustering
algorithm estimating the density cluster tree. Like Hartigan consistency,
minimality and separation are limit properties, and do not directly quantify the
disparity between the true density cluster tree and an estimate. We now
introduce a \emph{merge distortion metric} on cluster trees
(equipped with height functions) which will allow us to do just that. 

We make our definitions specifically so that convergence in the merge distortion
metric implies the desirable properties of minimality and separation.
Specifically, consider once again the density depicted in
Figure~\ref{fig:merge_height}. Suppose we run a cluster tree algorithm on a
finite sample $X_n$ drawn from $f$, obtaining a hierarchical clustering
$\hc{\hat C}_{f,n}$. Let $\hatT$ be this clustering equipped with $f$ as a
height function. We may then talk about the height at which two points merge in
$\hatT$, and of the level at which clusters are connected and separated in
$\hatT$. These are the concepts required to discuss minimality and separation.

Suppose that the algorithm ensures minimality and separation in the limit. What
can we say about the merge height of $a$ and $a'$ in $\hatT$ as $n \to \infty$?
First, minimality will suggest that $M \cap X_n$ be connected in $\hatT$ at
level $\mu - \delta$, with $\delta \to 0$, where $\mu$ is as it appears in
Figure~\ref{fig:merge_height}. This implies that the merge height of
$a$ and $a'$ is bounded below by $\mu - \delta$, with $\delta \to 0$. On the
other hand, separation implies that $A \cap X_n$ and $A' \cap X_n$ be separated
at level $\mu + \delta$, with $\delta \to 0$. Therefore the merge height of $a$
and $a'$ is bounded above by $\mu + \delta$, with $\delta \to 0$.  Hence in the
limit $n \to \infty$, the merge height of $a$ and $a'$ in $\hatT$, written
$m_\hatT(a,a')$, must converge to $\mu$, which is otherwise known as
$m_\T(a,a')$: the merge height of $a$ and $a'$ in the true density cluster tree.

With this as motivation, we'll work backwards, defining our distance between
clusterings in such a way that convergence in the metric implies that the merge
height between any two points in the estimated tree converges to the merge
height in the true density cluster tree. We'll then show that this entails
minimality and separation, as desired.

\paragraph{Merge distortion metric.}

Let $\hhc{C}_1 = (\X_1, \hc{C}_1, h_1)$ and $\hhc{C}_2 = (\X_2, \hc{C}_2, h_2)$
be two cluster trees equipped with height functions. Recall from
Definition~\ref{def:merge_height} that each cluster tree is associated with its
own merge height function which summarizes the level at which pairs of points
merge. We define the distance between $\hhc{C}_1$ and $\hhc{C}_2$ in terms of
the distortion between merge heights. In general, $\hhc{C}_1$ and
$\hhc{C}_2$ cluster different sets of objects, so we will
use the distortion with respect to a \emph{correspondence}\footnote{
    Recall that a correspondence $\gamma$ between sets $S$ and $S'$ is a subset
    of $S \times S'$ such that for $\forall s \in S, \exists s' \in S'$ such
    that $(s,s') \in \gamma$, and $\forall s' \in S', \exists s \in S$ such that
    $(s,s') \in \gamma$.
} between these sets.

\begin{definition}[Merge distortion metric]
Let $\hhc{C}_1 = (X_1, \hc{C}_1, h_1)$ and $\hhc{C}_2 = (X_2, \hc{C}_2, h_2)$ be
two hierarchical clusterings equipped with height functions. Let $S_1 \subset
X_1$ and $S_2 \subset X_2$. Let $\gamma \subset S_1 \times S_2$ be a
correspondence between $S_1$ and $S_2$. The merge distortion distance between
$\hhc{C}_1$ and $\hhc{C}_2$ with respect to $\gamma$ is defined as
\begin{equation*}
\label{eqn:correspondence_distance}
    d_\gamma(\hhc{C}_1, \hhc{C}_2) 
        = 
                \max_{(x_1,x_2), (x_1',x_2') \in \gamma}
                |m_{\hhc{C}_1}(x_1,x_1') - m_{\hhc{C}_2}(x_2,x_2')|.
\end{equation*}
\end{definition}

The above definition is related to the standard notion of the distortion of a
correspondence between two metric spaces \citep{burago_2001}.  We note that if
$X_1 = X_2$ and $\gamma$ is a correspondence between $X_1$ and $X_2$, then
$d_\gamma(\hhc{C}_1, \hhc{C_2}) = 0$ implies that $\hhc{C}_1 = \hhc{C}_2$ in the
sense that the two trees $\hc{C}_1$ and $\hc{C}_2$ are isomorphic and the height
function for corresponding nodes are identical. 

Now consider the special case of the distance between the true density cluster
tree $\T = (\X, \hc{C}_f, f)$ and a finite estimate. Suppose we run a
hierarchical clustering algorithm on a sample $X_n \subset \X$ of size $n$
drawn from $f$, obtaining a cluster tree $\hc{\hat C}_{f,n}$. Denote by $\hatT
= (X_n, \hc{\hat C}_{f,n}, f)$ the cluster tree equipped with height function
$f$.  Then the natural correspondence is induced by identity in $X_n$: That is,
$\hat \gamma_n = \{(x,x) : x \in X_n\}$.  We then define our notion of
convergence to the density cluster tree with respect to this correspondence:

\begin{definition}[Convergence to the density cluster tree]
\label{def:convergence}
We say that a sequence of cluster trees $\{\hatT\}$ converges to the high
density cluster tree $\T$ of $f$, written $\hatT \to \T$, if for any $\eps > 0$
there exists an $N$ such that for all $n \geq N$, $d_{\hat \gamma_n}(\hatT, \T)
< \eps$.
\end{definition}


\section{Properties of the merge distortion metric} 

We now prove various useful properties of our merge distortion metric. First,
we show that convergence in the distance implies both uniform minimality and
uniform separation. We then show that the converse is also true. We conclude by
discussing stability properties of the distance.

\begin{theorem}
$\hatT \to \T$ implies 1) uniform minimality and 2) uniform separation.
\end{theorem}
\begin{proof}
Our proof consists of two parts.

\emph{Part I:} $\hatT \to \T$ implies uniform minimality.
Pick any $\delta > 0$ and let $n$ be large enough that $d(\T, \hatT) < \delta$.
Let $A$ be a connected component of $\suplev$ for arbitrary $\lambda$. Let $a,a'
\in A \cap X_n$. Then $m_{\hatT}(a,a') > m_{\T}(a,a') - \delta$. But $a$
and $a'$ are elements of $A$, such that $m_\T(a,a') \geq \lambda$. Hence
$m_\hatT(a,a') > \lambda - \delta$. Since $a$ and $a'$ were arbitrary, it
follows that $A \cap X_n$ is connected at level $\lambda - \delta$.

\emph{Part II:} $\hatT \to \T$ implies uniform separation.
Pick any $\delta > 0$ and let $n$ be large enough that $d(\T, \hatT) < \delta$.
Let $A$ and $A'$ be disjoint connected components of $\suplev$ for arbitrary
$\lambda$. Let $\mu := m_\T(A \cup A')$ be the merge height of $A$ and $A'$ in
the density cluster tree. Take any $a \in A \cap X_n$ and $a' \in A' \cap X_n$.
Then $m_\hatT(a,a') < m_\T(a,a') + \delta = \mu + \delta$. Therefore $a$ and
$a'$ are separated at level $\mu + \delta$. Since $a$ and $a'$ were arbitrary,
it follows that $A \cap X_n$ and $A' \cap X_n$ are separated at level $\mu +
\delta$.
\end{proof}

The converse is also true. In other words, convergence in our metric is
equivalent to the combination of uniform minimality and uniform separation.

\begin{theorem}
\label{thm:uniform_implies_convergence}
If $\hatT$ ensures uniform separation and uniform minimality, then $\hatT \to
\T$.
\end{theorem}

\begin{proof}
Take any $\delta > 0$. Uniform separation and minimality imply that there
exists an $N$ such that for all $\lambda$ any cluster $A \in \suplev$ is
connected at level $\lambda - \delta$, and for all $\mu$ any two disjoint
clusters $B, B'$ merging at $\mu$ are separated at level $\mu + \delta$. Assume
$n \geq N$, and consider any $x,x' \in X_n$. W.L.O.G., assume $f(x') \geq f(x)$.
We will show that $|m_\hatT(x,x') - m_\T(x,x')| \leq \delta$.

Let $A$ be the connected component of $\{ f \geq f(x) \}$ containing $x$, and
let $A'$ be the connected component of $\{ f \geq f(x') \}$ containing $x'$.
There are two cases: either $A' \subseteq A$, or $A \cap A' = \emptyset$.

\emph{Case I:} $A' \subseteq A$. Then $m_\T(x,x') = f(x)$.
Minimality implies that $A \cap X_n$ is connected at level $f(x) - \delta$,
and therefore $m_\hatT(x,x') \geq f(x) - \delta$. On the other hand, clearly
$m_\hatT(x,x') \leq f(x)$. Hence $|m_\hatT(x,x') - m_\T(x,x')| \leq \delta$.

\emph{Case II:} $A \cap A' = \emptyset$. Let $\mu := m_\T(x,x')$ be the merge
height of $x$ and $x'$ in the density cluster tree of $f$, and suppose that
$M$ is the connected component of $\{ f \geq \mu \}$ containing $x$ and $x'$.
Then separation implies that $x$ and $x'$ are separated at level $\mu + \delta$,
such that $m_\hatT(x,x') < \mu + \delta$. On the other hand, minimality
implies that $M \cap X_n$ is connected at level $\mu - \delta$, so that
$m_\hatT(x,x') \geq \mu - \delta$. Therefore $|m_\hatT(x,x') - m_\T(x,x')| \leq
\delta$.
\end{proof}

\paragraph{Stability.}
\label{subsec:stability}

An important property to study for a distance measure is its stability; namely,
to quantify how much cluster tree varies as input is perturbed.  We provide two
such results. 

The first result says that the density cluster tree induced by a density
function is stable under our merge-distortion metric with respect to
$L_\infty$-perturbation of the density function.  The second result states that
given a fixed hierarchical clustering, the cluster tree is stable w.r.t. small
changes of the height function it is equipped with.  The proofs of these results
are in Appendix~\ref{apx:proofs}. 

\begin{theorem}[$L_\infty$-stability of true cluster tree]
\label{thm:stability-truetree}
Given a density function $f: \X \to \mathbb{R}$ supported on $\X \subset
\mathbb{R}^d$, and a perturbation $\tilde f: \X \to \mathbb{R}$ of $f$, let
$\hc{C}_f$ and $\hc{C}_{\tilde f}$ be the resulting density cluster tree as
defined in Definition \ref{def:cluster_tree}, and let $\hhc{C}_f := (\X,
\hc{C}_f, f)$ and $\hhc{C}_{\tilde f} := (\X, \hc{C}_{\tilde f}, \tilde f)$
denote the cluster tree equipped with height functions.  We have $d_\gamma
(\hhc{C}_f, \hhc{C}_{\tilde f}) \le \|f - \tilde f\|_\infty$, where $\gamma
\subset \X \times \X$ is the natural correspondence induced by identity $\gamma
= \{(x,x) \mid x\in \X\}$. 
\end{theorem}

\begin{theorem}[$L_\infty$-stability w.r.t. $f$]
\label{thm:stability1}
Given a  cluster tree $(X, \hc{C})$, let $\hhc{C}_1 = (X, \hc{C}, f_1)$ and
$\hhc(C)_2 = (X, \hc{C}, f_2)$ be the hierarchical clusterings equipped with two
height function $f_1$ and $f_2$, respectively. Let $\gamma: X \times X$ be the
natural correspondence induced by identity on $X$; that is, $\gamma = \{ (x, x)
\mid x\in X\}$.  We then have $d_\gamma (\hhc{C}_1, \hhc{C}_2) \le 2\| f_1 - f_2
\|_\infty$. 
\end{theorem}

Theorem \ref{thm:stability1} in particular leads to the following: Given a
density $f: \X \to \mathbb{R}$ supported on $\X \subset \mathbb{R}^d$, suppose
we have a hierarchical clustering $\hat{\hc{C}}_n$ constructed from a sample
$X_n \subset \X$. However, we do not know the true density function $f$.
Instead, suppose we have a density estimator producing an empirical density
function $\tilde f_n: X_n \to \mathbb{R}$.  Set $\hat{\hhc{C}}_{f,n} = (X_n,
\hat{\hc{C}}_n, f)$ as before, and $\tilde{\hhc{C}}_{\tilde f, n} = (X_n,
\hat{\hc{C}}_n, \tilde f_n)$.  Theorem \ref{thm:stability1} implies that $d (
\hat{\hhc{C}}_{f,n}, \tilde{\hhc{C}}_{\tilde f, n} ) \le \|f - \tilde
f_n\|_\infty$.  By the triangle inequality, this further bounds 
\begin{align}\label{eqn:empirical}
    d(\hc{C}_f, \tilde{\hhc{C}}_{\tilde f, n}) \le d(\hc{C}_f,
    \hat{\hhc{C}}_{f,n}) + \|f - \tilde f_n\|_\infty. 
\end{align}
Assuming that the density estimator is consistent, we note that the cluster tree
$\tilde{\hhc{C}}_{\tilde f, n}$ also converges to $\hc{C}_f$ if
$\hat{\hhc{C}}_{f,n}$ converges to $\hc{C}_f$. 
This has an important implication from a practical point of view. Imagine that
we are given a sequence of more and more samples $X_{n_1}, X_{n_2}, \ldots$, and
we construct a sequence of hierarchical clusterings $\hat{\hc{C}}_{n_1},
\hat{\hc{C}}_{n_2}, \ldots$.  In practice, in order to test whether the current
hierarchical clustering converges or not, one may wish to compare two
consecutive clusterings $\hat{\hc{C}}_{n_i}$ and $\hat{\hc{C}}_{n_{i+1}}$ and
measure their distance. However, since the true density is not available, one
cannot compute the cluster tree distance $d_{\gamma_{n_i}} (\hat{\hhc{C}}_{f,
n_i}, \hat{\hhc{C}}_{f, n_{i+1}})$, where the correspondence is induced by the
natural inclusion from $X_{n_i} \subseteq X_{n_{i+1}}$, that is, $\gamma_{n_i} =
\{ (x, x) \mid x\in X_{n_i}\}$. Eqn. (\ref{eqn:empirical}) justifies the use of
a consistent empirical density estimator and computing $d_{\gamma_{n_i}}
(\tilde{\hhc{C}}_{\tilde f, n_i}, \tilde{\hhc{C}}_{\tilde f, n_{i+1}})$ instead. 


\section{Convergence of robust single linkage} 
\label{section:chaudhuri}

We now analyze the robust single linkage algorithm of \citet{chaudhuri_2010} in
the context of our formalism. \cite{chaudhuri_2010} and \cite{chaudhuri_2014}
previously studied the sense in which robust single linkage ensures that
clusters are separated and connected at the appropriate levels of the empirical
tree. Our analysis translates their results to our definitions of minimality
and separation, thereby reinterpreting the convergence of robust single linkage
in terms of our merge distortion metric.

A simple description of the algorithm is given in
Appendix~\ref{section:robust_single_linkage_description}. Essentially, the
method produces a sequence of graphs $G_r$ as $r$ ranges from $0$ to $\infty$.
The sequence has a nesting property: if $r \leq r'$, then $V_r \subset V_{r'}$
and $E_r \subset E_r'$. We interpret this sequence of graphs as a cluster
tree by taking each connected component in any graph $G_r$ as a cluster.  We
equip this cluster tree with the true density $f$ as a height function, and
refer to it as $\hatT$ in conformity with the preceding sections of this paper.

In what follows, assume that $f$ is: 
1) $\mylip$-Lipschitz;
2) compactly supported (and hence bounded from above); and
3) such that $\{ f \geq \lambda \}$ has finitely-many connected components for
any $\lambda$.  
We will prove that the algorithm ensures minimality and separation. This,
together with the assumptions on $f$ and
Theorem~\ref{thm:uniform_implies_convergence}, will imply convergence in the
merge distortion distance.

Suppose we run the robust single linkage algorithm on a sample of size $n$.
Denote by $v_d$ the volume of the $d$-dimensional unit hypersphere, and let
$B(x,r)$ the closed ball of radius $r$ around $x$ in $\R^d$. We will write
$f(B(x,r))$ to denote the probability of $B(x,r)$ under $f$. Define $r(\lambda)$
to be the value of $r$ such that $v_d r^d \lambda = \knplus$.  Here, $k$ is a
parameter of the algorithm which we will constrain, and $C_\delta$ is the
constant appearing in the Lemma~IV.1 of \citet{chaudhuri_2014}.  First, we must
show that in the limit, $G_{r(\lambda)}$ contains no points of density less
than $\lambda - \epsilon$, for arbitrary $\epsilon$.

\begin{lemma}
\label{lemma:connected_lower_bound}
Fix $\epsilon > 0$ and $\lambda \geq 0$. Then if $\alpha \geq \sqrt{2}$ and $k
\geq \left(8 C_\delta \lambda/\epsilon\right)^2 d \log n$, there exists an $N$
such that for all $n \geq N$, if $x \in G_{r(\lambda)}$, then $f(x) > \lambda -
\epsilon$.
\end{lemma}

\begin{proof}
\newcommand{\reps}{{\tilde r}}
Define $\reps = r(\lambda - \epsilon / 2)$. There
exists an $N$ such that for any $n \geq N$, $\reps \mylip \leq \epsilon / 4$.
Consider any point $x \in G_{\reps}$.  By virtue of $x$'s membership in the
graph, $X_n$ contains $k$ points within $B(x, \reps)$. Lemma IV.1 in
\citep{chaudhuri_2014} implies that $f(B(x,\reps)) > \frac{k}{n} -
\frac{C_\delta}{n}\sqrt{k d \log n}$.  From our smoothness assumption, we have
$v_d \reps^d (f(x) + \reps \mylip) \geq f(B(x,\reps)) > \frac{k}{n} -
\frac{C_\delta}{n} \sqrt{k d \log n}$. Multiplying both sides by $\lambda -
\epsilon/2$ and substituting gives:
$
    \textstyle
    v_d \reps^d (\lambda - \epsilon/2)(f(x) + \reps \mylip) 
        = \left( \knplus \right) (f(x) + \reps \mylip)
        > (\lambda - \epsilon/2) ( \knminus )
$
so that

\begin{align*}
\textstyle
    f(x) > (\lambda - \epsilon/2)
            \left\{ 
                \frac{k - C_\delta \sqrt{k d \log n}}
                     {k + C_\delta \sqrt{k d \log n}} \right\} 
                - \reps \mylip
    &\geq \left(1 - 2 \frac{C_\delta \sqrt{d \log n}}{\sqrt{k}}\right)
               (\lambda - \epsilon/2) - \epsilon / 4\\
    &\geq \left(1 - \frac{\epsilon}{4\lambda}\right)(\lambda - \epsilon/2) 
               - \epsilon/4
    \geq \lambda - \epsilon
\end{align*}
Hence for any point $x \in G_{\reps}$, $f(x) > \lambda - \epsilon$.
Note that $\reps > r(\lambda)$, implying that any point in $G_{r(\lambda)}$ is
also in $G_\reps$. Therefore if $x \in G_{r(\lambda)}$, $f(x) > \lambda -
\epsilon$.
\end{proof}

We now make our claim. We will use the following fact without proof: For any $A
\in \{f \geq \lambda\}$ and $\delta > 0$, there exists an $N$ such that for all
$n \geq N$, if $A \cap X_n \neq \emptyset$, there is at least one point $x \in A
\cap X_n$ with $f(x) < \lambda + \delta$.  This follows immediately from the
continuity of $f$ and the inequalities in the Lemma IV.1 of
\citet{chaudhuri_2014}.

\begin{theorem}
Robust single linkage converges in probability to the density cluster tree $\T$
in the merge distortion distance.
\end{theorem}
\begin{proof}\newcommand{\A}{{\tilde A}}
It is sufficient to prove minimality and separation, as then
Theorem~\ref{thm:uniform_implies_convergence} will imply convergence.  Fix any
$\eps > 0$, and let $A$ be a connected component of $\{ f \geq \lambda \}$.
Define $\sigma = \eps/(2\mylip)$, and let $A_\sigma$ be the set $A$ thickened
by closed balls of radius $\sigma$. Define $\lambda' := \inf_{x \in
A_\sigma} f(x) \geq \lambda - \eps / 2$.  Theorem IV.7 in
\citep{chaudhuri_2014} implies that there exists an $N_1$ such that for all $n
\geq N_1$, $A \cap X_n$ is connected in $G_{r(\lambda')}$.  Take $\epsilon =
\eps/2$ in our Lemma~\ref{lemma:connected_lower_bound}; there exists an
$N_2$ above which each point $x$ in $G_{r(\lambda')}$ has density $f(x) >
\lambda' - \epsilon \geq (\lambda - \eps /2) - \eps/2 = \lambda - \eps$.
Then for all $n \geq \max \{N_1,N_2\}$, $A \cap X_n$ is connected
in $G_{r(\lambda')}$ at level no less than $\lambda - \eps$. This proves
minimality.

Again fix $\eps > 0$ and let $A$ and $A'$ be connected components of $\{ f \geq
\lambda \}$ merging at some height $\mu = m_\T(A \cup A')$. Let $\A$ and $\A'$
be the connected components of $\{ f \geq \mu + \eps/2 \}$ containing $A$ and
$A'$, respectively. Define $\sigma = \eps/(4\mylip)$, and let $\A_\sigma$ (resp.
$\A'_\sigma$) be the set $\A$ (resp. $\A'$) thickened by closed balls of radius
$\sigma$. Define $\mu' := \inf_{x \in \A_\sigma \cup \A'_\sigma} f(x) \geq \mu +
\eps/4$. Then Lemma IV.3 in \citep{chaudhuri_2014} implies\footnote{ More
    precisely, Lemma IV.3 requires $A$ and $A'$ to be so-called $(\sigma,
    \epsilon)$-separated, for some $\sigma$ and $\epsilon$. It follows from the
    Lipschitz-continuity of $f$ that there is some $\epsilon$ so that $A$ and
    $A'$ are $(\sigma, \epsilon)$-\emph{separated} for this choice of $\sigma$.
} that there exists some $N_1$ such that for all $n \geq N_1$, $\A \cap X_n$ and
$\A' \cap X_n$, are disconnected in $G_r(\mu')$ and individually connected. Let
$N_2$ be large enough that there exists a point $x_1 \in \A \cap X_n$ with
$f(x_1) < \mu + \eps$. Then for all $n \geq \max\{N_1,N_2\}$, $A \cap X_n$ and
$A' \cap X_n$ are separated at level $\mu + \eps$. This proves separation.
\end{proof}


\section{Split-tree based hierarchical clustering} 

\label{section:split_tree}

We also consider a different approach to estimate the cluster tree, using
ideas from the field of computational topology. The method is
based on the  clustering algorithm proposed by \citet{CGOS13}, while that work
focuses on analyzing flat clustering. We briefly describe our method and state
our main result. A detailed description and the proof are relegated to
Appendix~\ref{appendix:sec:mergetree}.

Our algorithm takes as input a set of points $P_n$ sampled iid from a density
$f$ supported on an unknown Riemannian manifold, an empirically-estimated
density function $\tilde f_n$, and a parameter $r>0$, and outputs a hierarchical
clustering tree $\myT$ on $P_n$.  Let $K_n$ be the proximity graph on $P_n$, in
which every point in $P_n$ is connected to every other point that is within
distance $r$. We then track the connected components of the subgraph of $K_n$
spanned by all points $P_n^\lambda = \{p \in P_n : \tilde f_n(p) \geq \lambda
\}$ as we sweep $\lambda$ from high to low. The set of clusters (connected
components in the subgraphs) produced this way and their natural nesting
relations give rise to a hierarchical clustering that we refer to as the
\emph{split-cluster} tree $\myT$. 

Comparing this with the definition of high-density cluster tree in Definition
\ref{def:high_density_cluster_tree}, we note that intuitively, the \splittree{}
$\myT$ is a discrete approximation of the high-density cluster tree $\hc{C}_f$
for the true density function $f: \M \to \mathbb R$ where (i) the density $f$ is
approximated by the empirical density $\tilde f_n$; and (ii) the connectivity of
the domain $\M$ is approximated by the proximity graph $K_n$. 

It turns out that the constructed tree $\myT$ is related to the so-called
\emph{split tree} studied in the  computational geometry and topology literature
as a variant of the \emph{contour tree}; see e.g, \citep{CSA03,WWW14}. Due to
this relation, the \splittree{} can be constructed efficiently in $O(n\alpha
(n))$ time using a union-find data structure once nodes in $P_n$ are sorted
\citep{CSA03}. 

Our main result is that under mild conditions, the \splittree{} converges to
the true high-density cluster tree $\hhc{C}_f$ of $f: \M \to \mathbb R$ in
merge distortion distance. See Appendix \ref{appendix:sec:mergetree} for full
details.  
\begin{theorem}\label{thm:splittree} Let $\M$ be a compact
$m$-dimensional Riemannian manifold embedded in $\mathbb{R}^d$ with bounded
absolute curvature and positive strong convexity radius. Let $f: \M \to \mathbb
R$ be a $\mylip$-Lipschitz probability density function supported on $\M$.  Let
$P_n$ be a set of $n$ points sampled i.i.d. according to $f$.  Assume that we
are given a density estimator such that $\| f - \tilde f_n\|_\infty$ converges
to $0$ as $n \to \infty$.  For any fixed $\eps > 0$, we have, with probability
$1$ as $n \to\infty$, that $d(\optT, \HmyT) \le (4\mylip+1) \eps$, where the
parameter $\myr$ in computing the \splittree{} $\myT$ is set to be $2\eps$, and
$\HmyT = (P_n, \myT, f)$ is the hierarchical clustering tree $\myT$ equipped
with the height function $f$.
\end{theorem}

\paragraph{Acknowledgements.}
The authors thank anonymous reviewers for insightful comments. This work is in
part supported by the National Science Foundation (NSF) under grants
CCF-1319406, RI-1117707, and CCF-1422830.
\enlargethispage{1.2em}


\bibliography{citations}

\appendix

\section{Proofs} 
\label{apx:proofs}

\subsection{Proof of Theorem~\ref{thm:nice_function_implies_uniform}}
\label{apx:proof:nice_function_implies_uniform}
\begin{theorem}
Let $f$ be a density supported on $\X$, and let $\{\hatT\}$ be a sequence of
cluster trees computed from finite samples $X_n \subset \X$. Suppose $f \leq M$
for some $M \in \R$, and that for any $\lambda$, $\suplev$ contains finitely
many connected components. Then
\begin{enumerate}
    \item If $\{\hatT\}$ ensures minimality for $f$, it ensures uniform
        minimality.
    \item If $\{\hatT\}$ ensures separation for $f$, it ensures uniform
        separation.
\end{enumerate}
\end{theorem}

\begin{proof}
We will prove the first case, in which $\hatT$ ensures minimality. The proof
of uniform separation follows closely, and is therefore omitted.

Pick $\delta > 0$. Let $\T(\lambda)$ denote the (finite) set of connected
components of $\suplev$. Consider the collection of connected components of
superlevel sets spaced $\delta/2$ apart:
\newcommand{\D}{\hc{D}}
\[
    \D = \bigcup_{n=0}^{\lfloor 2M/\delta \rfloor} \T(n\delta/2)
\]

The fact that $\hatT$ ensures minimality implies that  for each $C \in \D$
there exists an $N(C)$ such that for all $n \geq N(C)$, $C \cap X_n$ is
connected at level $h(C) - \delta/2$. Let $N = \max_{C \in \D} N(C)$. This is
well-defined, as $\D$ is a finite set.

Let $A$ be a connected component of $\suplev$ for an arbitrary $\lambda$. Let
$\lambda' = \lfloor 2 \lambda / \delta \rfloor \frac{\delta}{2}$, i.e.,
$\lambda'$ is the largest multiple of $\delta/2$ such that $\lambda' \leq
\lambda$. Then $A$ is a subset of some connected component $A'$ of
$\suplev[\lambda']$. Note that $A' \in \D$, so that $A' \cap X_n$ is connected
at level $\lambda' - \delta / 2$. Therefore $A \cap X_n$ is connected at level
$\lambda' - \delta / 2 > (\lambda - \delta / 2) - \delta/2 = \lambda - \delta$.
Since $A$ was arbitrary, and the choice of $N$ depended only upon $\delta$, it
follows that $\hatT$ ensures uniform minimality.
\end{proof}

\subsection{Proof of Theorem~\ref{thm:stability-truetree}}
\label{apx:stability-truetree}
\begin{proof}
Set $\delta = \|f - \tilde f\|_\infty$. Let $x, x'$ be two arbitrary points
from $X$. We need to show that $| d_{\hhc{C}_f} (x, x') - d_{\hhc{C}_{\tilde
f}}(x, x')| \le 4\delta$, which will then implies the theorem.  In what
follows, we prove that $d_{\hhc{C}_f}(x, x') \le d_{\hhc{C}_{\tilde f}}(x, x')
+ 4\delta$. 

Let $m = m_{\hhc{C}_f}(x, x')$ denote the merge height of $x$ and $x'$ w.r.t.
$\hhc{C}_f$. This means that there exists a connected component $C \in \{ y \in
\X \mid f(y) \ge m\}$ such that $x, x' \in C$.  Since $\|f - \tilde f\|_\infty
= \delta$, we have that for any point $y\in C$, $|\tilde f(y) - f(y) | \le
\delta$ and thus $\tilde f(y) \ge m- \delta$.  Hence all points in $C$ must
belong to the same connected component, call it $\tilde C (\supseteq C) \in \{y
\in \X \mid \tilde f(y) \ge m - \delta \}$ with respect to the clustering
$\hc{C}_{\tilde f}$.  It then follows that the merge height $m_{\hhc{C}_{\tilde
f}}(x, x') \ge m - \delta$.  Combining this with that $\| f - \tilde f\|_\infty
= \delta$, we have: 
\begin{align*}
    d_{\hhc{C}_{\tilde f}}(x, x') &= \tilde f(x) + \tilde f(x') - 2
    m_{\hhc{C}_{\tilde f}}(x,x') \\ &\le f(x) + \delta + f(x') + \delta - 2m +
    2\delta = d_{\hhc{C}_f}(x,x') + 4\delta. 
\end{align*}
The proof for $d_{\hhc{C}_f}(x, x') \le d_{\hhc{C}_{\tilde f}}(x, x') + \delta$
is symmetric. The theorem then follows. 
\end{proof}

\subsection{Proof of Theorem~\ref{thm:stability1}}
\label{apx:stability1}
\begin{proof}
Set $\delta := \|f_1 - f_2\|_\infty$. Let $x, x'$ be two arbitrary points from
$X$. We need to show that $| d_{\hhc{C}_1} (x, x') - d_{\hhc{C}_2}(x, x')| \le
4\delta$, which will then implies the theorem.  In what follows, we prove that
$d_{\hhc{C}_2}(x, x') \le d_{\hhc{C}_1}(x, x') + 4\delta$. 

Let $m_1 = m_{\hhc{C}_1}(x, x')$ denote the merge height of $x$ and $x'$ w.r.t.
$\hhc{C}_1$. This means that there exists a cluster $C \in \hc{C}$ such that
$x, x' \in C$ and $f_1(C) = m_1$.  Since $f_i(C) = \min_{y \in C} f_i(y)$, for
$i = 1, 2$, we thus have that $f_2(C) \in [m_1 - \delta, m_1+\delta]$.  It then
follows that $m_{\hhc{C}_2}(x,x') \ge f_2(C) \ge m_1 - \delta$.  Combining with
that $\|f_1 - f_2\|_\infty = \delta$, we have: 
\begin{align*}
    d_{\hhc{C}_2}(x, x') &= f_2(x) + f_2(x') - 2 m_{\hhc{C}_2}(x,x') \\ &\le
    f_1(x) + \delta + f_1(x') + \delta - 2m_1 + 2\delta = d_{\hhc{C}_1}(x,x') +
    4\delta. 
\end{align*}
The proof for $d_{\hhc{C}_1}(x, x') \le d_{\hhc{C}_2}(x, x') + \delta$ is
symmetric. The theorem then follows. 
\end{proof}


\section{Robust single linkage} 
\label{section:robust_single_linkage_description}

We briefly describe the robust single linkage algorithm, and refer readers to
the work of \citet{chaudhuri_2010} and \citet{chaudhuri_2014} for details. In
what follows, let $B(x,r)$ denote the closed ball of radius $r$ around $x$.

The algorithm operates as follows:
Given a sample $X_n$ of $n$ points drawn from a density $f$ supported on $\X$,
and parameters $\alpha$ and $k$, perform the following steps:
\begin{enumerate}
\item For each $x_i \in X_n$, set $r_k(x_i) = \min \{ r : B(x_i, r) \text{
    contains $k$ points}\}$.
\item As $r$ grows from 0 to $\infty$:
    \begin{enumerate}
        \item Construct a graph $G_r$ with nodes $\{x_i : r_k(x_i) \leq r\}$.
            Include edge $(x_i,x_j)$ if $\|x_i - x_j\| \leq \alpha r$.
        \item Let $\mathbb C_n(r)$ be the connected components of $G_r$.
    \end{enumerate}
\end{enumerate}

The algorithm produces a series of graphs as $r$ ranges from 0 to $\infty$. Each
connected component in $G_r$ for any $r$ is considered a cluster. The clusters
exhibit hierarchical structure, and can be interpreted as a cluster tree.  We
may therefore discuss the sense in which this discrete tree converges to the
ideal density cluster tree.

\smallskip
\smallskip


\section{Split-tree based hierarchical clustering} 
\label{appendix:sec:mergetree}

In this section we inspect a different approach, based on the one proposed
and studied by \cite{CGOS13}, to obtain a hierarchical clustering for points
sampled from a density function supported on a Riemannian manifold, using tools
from the emerging field of computational topology. 

In particular, we focus on the following setting: Let $\M \subseteq
\mathbb{R}^d$ be a  smooth $m$-dimensional Riemannian manifold $\M$ embedded in
the ambient space $\mathbb R^d$, and $f: \M \to \mathbb R$ a $\mylip$-Lipschitz
probability density function supported on $\M$.  Let $P_n$ denote a set of $n$
points sampled i.i.d. according to $f$.  We further assume that we have a
density estimator $\tilde f_n : P_n \to \mathbb R$ which estimates the true
density $f$ with the guarantee that $\| f - \tilde f_n \|_\infty \le \fe(n)$ for
an error function $\fe(n)$ which tends to zero as $n \to +\infty$. 

\subsection{Split-cluster tree construction.}

We now describe an algorithm which takes as input $P_n$ and the empirical
density function $\tilde f_n: P_n \to \mathbb{R}$, and outputs a hierarchical
clustering tree $\myT$ on $P_n$.  The algorithm uses a parameter $\myr > 0$,
which intuitively should go to zero as $n$ tends to infinity. 

Let $K_n = (P_n, E)$ denote the 1-dimensional simplicial complex, where $E := \{
(p, p') \mid \|p- p'\| \le \myr \}$. In other words, $K_n$ is the proximity
graph on $P_n$ where every point in $P_n$ is connected to all other points from
$P_n$ within $r$ distance to it.  We now define the following hierarchical
clustering (cluster tree) $\myT$: 

Given any value $\lambda$, let $P_n^{\lambda}:=\{ p\in P_n \mid \tilde f_n(p)
\ge \lambda\}$ be the set of vertices with estimated density at least $\lambda$,
and let $K_n^\lambda$ be the subgraph of $K_n$ induced by $P_n^\lambda$.  The
subgraph $K_n^\lambda$ may have multiple connected components, and the vertex
set of each connected component gives rise to a cluster.  The collection of such
clusters for all $\lambda\in \mathbb{R}$ is $\myT$, which we call the
\emph{\splittree{}} of $P_n$ w.r.t. $\tilde f_n$.  We put the parameter $\myr$
in $\myT$ to emphasize the dependency of this cluster tree on $\myr$. 

In particular, note that the function $\tilde f_n : P_n \to \mathbb R$ induces a
piecewise-linear (PL) function on the underlying space $|K_n|$ of $K_n$, which
we denote  as $\tilde f: |K_n | \to \mathbb R$.  It turns out that the tree
representation of this cluster tree $\myT$ is exactly the so-called \emph{split
tree} of this PL function $\tilde f$ as studied in the literature of
computational geometry and topology, as a variant of the contour tree; see e.g,
\citep{CSA03,WWW14}. This is why we refer to $\myT$ as \splittree{} of $P_n$.
The \splittree{} $\myT$ can be easily computed in $O(n\alpha(n))$ time using the
union-find data structure, once the vertices in $P_n$ are already sorted
\citep{CSA03}. 

We note that \cite{CGOS13} proposed a clustering algorithm based on this idea,
and provided various nice theoretical studies of \emph{flat clusterings}
resulted from such a construction. We instead focus on the hierarchical
clustering tree constructed using this split tree idea.  

Finally, recall that $f: \M \to \mathbb{R}$ is the true density function.  Given
$\myT$, let $\hat{\hhc{C}}_{f,n} = (P_n, \myT, f)$ be the corresponding cluster
tree  equipped with height function $f: P_n \to \mathbb{R}$ (which is the
restriction of $f$ to $P_n$).  As before, we still use $\hc{C}_f$ to denote the
high-density cluster tree w.r.t. the true density function $f$, and $\hhc{C}_f =
(\M, \hc{C}_f, f)$ be the corresponding cluster tree equipped with height
function $f: \M \to \mathbb{R}$.

In what follows, we will study the convergence of the distance $d_\gamma (\optT,
\HmyT)$, where $\gamma: M \times P_n$ is the natural correspondence induced by
identity in $P_n$, that is, $\gamma = \{ (p, p) \mid p\in P_n$ (thus $p\in \M)
\}$.  For simplicity of presentation, we will omit the reference of this natural
correspondence $\gamma$ in the remainder of this section. 

\subsection{Convergence of \splittree}

First, we introduce some notation.  Let $d(x,y)$ denote the Euclidean distance
between any two points $x, y\in \mathbb{R}^d$, while $d_M(x, y)$ denotes the
geodesic distance between points $x, y\in \M$ on the manifold $\M$.  Given a
smooth manifold $\M$ embedded in $\mathbb{R}^d$, the \emph{medial axis} of $\M$,
denoted by $\mathcal{A}_M$, is the set of points in $\mathbb{R}^d$ which has
more than one nearest neighbor in $\M$.  The \emph{reach of $\M$}, denoted by
$\rho(\M)$, is the infimum of the closest distance from any point in $\M$ to the
medial axis, that is, $\rho(\M) = \inf_{x\in \M} d(x, \mathcal{A}_M)$. 

Following the notations of \cite{CGOS13}, we further define:
\begin{definition}[(Geodesic) $\eps$-sample] Given a subset $Y \subseteq \M$ and
a parameter $\eps > 0$, a set of points $Q \subset Y$ is a \emph{(geodesic)
$\eps$-sample of $Y$} if every point  of $Y$ is within $\eps$ geodesic distance
to some point in $Q$; that is, $\forall x \in Y, \min_{q\in Q} d_\M (x, q) \le
\eps$.  \end{definition}

In what follows, let $\M^\lambda = \{ x \in \M \mid f(x) \ge \lambda\}$ be the
super-level set of $f: \M\to \mathbb{R}$ w.r.t. $\lambda$. 

\begin{lemma}\label{thm:epssample-treedistance} We are given an $m$-dimensional
smooth manifold $\M \subset \mathbb{R}^d$ with a $\mylip$-Lipschitz density
function $f: \M \to \mathbb{R}$ on $\M$. Let $\rho(\M)$ be the reach of $\M$.
Let $P_n$ be an $\eps$-sample of $\M^\lambda$.  Assume that $\|f - \tilde f_n
\|_\infty \le \eta$ for $\tilde f_n: P_n \to \mathbb R$, and that the parameter
we use to construct $\myT$ satisfies $\myr \geq 2\eps$ and $\myr < \rho(\M)/2$.
Then $d(\optT, \hat{\hhc{C}}_{f, n}) \le \max \{ \mylip \myr + 2\eta, \lambda
\}$.  \end{lemma} \begin{proof} Consider any two points $p, p' \in P_n$.  Let
$m$ and $\hat{m}$ denote the merge height of $p$ and $p'$ in $\optT$ and in
$\HmyT$, respectively.  By definition, we have that
\begin{align}\label{eqn:treedistbound} 
 d(\optT, \hat{\hhc{C}}_{f,n}) =
\max_{p, p' \in P_n} |m_{\optT}(p, p') - m_{\HmyT}(p, p')|.\end{align}

We now distinguish two cases: \begin{description} \item[Case 1:] $m \ge
\lambda$.  \end{description} In this case, by definition of the merge height $m$
of $p$ and $p'$, we know that: (1) $f(p), f(p') \ge m$, and thus both $p$ and
$p'$ are from $P_n \cap \M^\lambda$, and (2) $p$ and $p'$ are connected in
$\M^m$, thus there is a path $\mygamma \subset \M^\lambda$ connecting $p$ and
$p'$ such that for any $x \in \mygamma$, $f(x) \ge m$.  We now show that the
merge height of $p$ and $p'$ in $\HmyT$ satisfies $\hat{m} \in [m - \mylip \myr
- 2\eta, m + \mylip \myr + 2\eta]$. 

Indeed, let $\pi: \M^\lambda \to P_n$ be the projection map that sends any $x\in
\M^\lambda$ to its nearest neighbor in $P_n$.  Since $P_n$ is an $\eps$-sample
for $\M^\lambda$, we have that $d_M(x,\pi(x)) \le \eps$ and thus $|f(x) -
f(\pi(x))| \le \mylip \eps$ (as $f$ is $\mylip$-Lipschitz), for any $x\in
\M^\lambda$.  Consider any two sufficiently close points $x, x' \in \mygamma$
(i.e, $\|x - x'\| < \myr - 2\eps$), we have that (1) either $\pi(x) = \pi(x')$,
(2) or $\pi(x) \neq \pi(x')$ but $$\| \pi(x) - \pi(x')\| \le \| \pi(x) - x \| +
\|x - x'\| + \|x' - \pi(x')\| < \myr .$$ 

In other words, $\pi(\mygamma)$ consists of a sequence of vertices $p= q_1, q_2,
\ldots, q_s = p' $ in $K_n$ such that there is an edge in $K_n$ connecting any
two consecutive $q_i, q_{i+1}$, $i\in [1, s-1]$. The concatenation of these
edges forms a path $\mygamma'= \langle p= q_1, q_2, \ldots, q_s = p' \rangle$ in
$K_n$ connecting $p$ to $p'$.  
Since for each $i\in [1, s-1]$, $q_i = \pi(x)$ for some $x\in \mygamma$, we have
that $$f(q_i) = f(\pi(x)) \ge f(x) - \mylip \eps \ge m - \mylip \eps, $$ where
the two inequalities follow from that $|f(x) - f(\pi(x))| \le \mylip \eps$ and
$f(x) \ge m$ for any $x \in \mygamma$.  Since $\|f - \tilde f_n \| \le \eta$, it
then follows that $\tilde f(q_i) \ge m - \mylip \eps - \eta$ for any $q_i \in
\mygamma'$. 

Recall that $K_n^\alpha$ denotes the subgraph of $K_n$ induced by the set of
points $P_n^\alpha = \{p\in P_n \mid \tilde f_n(p) \ge \alpha \}$ whose function
value w.r.t. the empirical density function $\tilde f_n$ is at least $\alpha$.
It then follows that $p$ and $p'$ should be connected in $K_n^\alpha$ for
$\alpha = m - \mylip \eps - \eta$.  It then follows that the merge height
$$\hat{m} = m_{\HmyT} (p, p') \ge \min_{q \in K_n^\alpha} f(q) \ge \min_{q \in
K_n^\alpha} \tilde f_n(q) - \eta \ge \alpha - \eta = m - \mylip \eps - 2\eta. $$

We now show the other direction, namely $m \ge \hat{m} - \mylip \myr- 2\eta$.
Indeed, by definition of $\hat{m}$, there is a path $\tilde \mygamma = \langle
\tilde q_1 = p, \tilde q_2, \ldots, \tilde q_t = \tilde p'\rangle$ in $K_n$
connecting $p$ and $p'$ such that for any $\tilde q_i \in \tilde \mygamma$,
$f(\tilde q_i) \ge \hat{m}$.  Now let $\ell_\M(q, q')$ denote a minimizing
geodesic between two points $q, q' \in \M$. We then have that there is a path
$$\tilde \mygamma' := \ell_\M(\tilde q_1, \tilde q_2) \circ \ell_\M(\tilde q_2,
\tilde q_3) \circ \cdots \circ \ell_\M(\tilde q_{t-1}, \tilde q_t)$$ in $\M$
connecting $p = \tilde q_1$ to $p' = \tilde q_t$. 

At the same time, note that for any two consecutive nodes $\tilde q_i$ and
$\tilde q_{i+1}$ from $\tilde \mygamma$, we know $\| \tilde q_i - \tilde q_{i+1}
\| \le \myr$ as $(\tilde q_i, \tilde q_{i+1})$ is an edge in $K_n$.  For $\myr <
\rho(\M)/2$, where $\rho(\M)$ is the reach of the manifold $\M$, by Proposition
1.2 of \cite{DSW11}, we have that the geodesic distance $d_\M(\tilde q_i, \tilde
q_{i+1})$ is at most $\frac{4}{3} \|\tilde q_i - \tilde q_{i+1} \|$. Thus $d_\M
(\tilde q_i, \tilde q_{i+1}) \le \frac{4}{3}\myr$.  In particular, for any point
$x\in \ell_\M (\tilde q_i, \tilde q_{i+1})$, it is within $\frac{2}{3}\myr$
distance to either $\tilde q_i$ or $\tilde q_{i+1}$. Hence we have that $$f(x)
\ge \min \{f(\tilde q_i), f(\tilde q_{i+1}) \} - \frac{2}{3}\mylip \myr >
\hat{m} - \mylip \myr ~~~\Rightarrow~~~m = m_\optT(p, p') \ge \min_{x\in \tilde
\mygamma'} f(x) \ge \hat{m} - \mylip \myr.$$

Putting everything together, we have that $| m - \hat{m} | \le \mylip \myr +
2\eta$ for the case $m \ge \lambda$. 

\begin{description}
\item[Case 2:] $m < \lambda$. 
\end{description}
First, note that the proof of $m \ge \hat{m} - \mylip \myr$ holds regardless of
the value of $m$. Hence we have $\hat{m} - m \le \mylip \myr $ for the case $m <
\lambda$ as well.
On the other hand, since $m < \lambda$, $m - \hat{m} < \lambda$.  
Thus $|\hat m - m | \le \max \{ \mylip \myr, \lambda \}$. 

The lemma follows from combining these two cases with Eqn.
(\ref{eqn:treedistbound}).  \end{proof}

\paragraph{Remark:} The bound in the above result can be large if the value
$\lambda$ is large. We can obtain a stronger result for points in $P_n \cap
\M^\lambda$ which is independent of $\lambda$. However, the above result is
cleaner to present and it suffices to prove our main convergence result in
Theorem \ref{thm:splittree}. 

To obtain a convergence result, we need to incur the following results from
\cite{CGOS13}. 
\begin{definition}[\cite{CGOS13}]\label{def:coveringandvolume}
Let $\M$  be an $m$-dimensional Riemannian manifold with intrinsic metric
$d_\M$.  Given a subset $A \subseteq \M$ and a parameter $r > 0$, define
$\mathcal V_r (A)$ to be the infimum of the Hausdorff measures achieved by
geodesic balls of radius $r$ centered in $A$; that is:
\begin{align}\label{eqn:volume} \mathcal V_r(A) = \inf_{x\in A} \mathcal{H}^m
(B_\M (x, r)),~~\text{where}~~B_\M(x,r) := \{ y\in \M \mid d_\M (x, y) \le r \}.
\end{align}
We also define the \emph{$r$-covering number of $A$}, denoted by
$\mathcal{N}_r(A)$ to be the minimum number of closed geodesic balls of radius
$r$ needed to cover $A$ (the balls do not have to be centered in $A$). 
\end{definition}

\begin{theorem}[Theorem 7.2 of \cite{CGOS13}]
Let $\M$ be an $m$-dimensional Riemannian manifold and $f: \M \to \mathbb R$ a
$\mylip$-Lipschitz probability density function. Consider a set $P$ sampled
according to $f$ in i.i.d. fashion. Then, for any parameter $\eps > 0$ and
$\alpha > \mylip \eps$, we are guaranteed that $P$ forms an $\eps$-sample of
$\M^\alpha$ with probability at least $1 - \mathcal{N}_{\eps/2} (\M^\alpha)
e^{-n (\alpha - \mylip \eps) \mathcal{V}_{\eps/2}(\M^\alpha)}$. 
\label{thm:epssample}
\end{theorem}

\paragraph{Remarks.} For simplicity, we now focus on the case where $\M$ is a
compact smooth embedded manifold with bounded absolute sectional curvature and
positive strong convexity radius $\rho_c(\M)$.  It follows from the
G\"{u}nther-Bishop Theorem that (see e.g, Appendix B of \cite{B15}) in this
case, there exists a constant $\mu$ depending only on the intrinsic property of
$\M$ such that $\mathcal V_r(\M^\alpha) \ge \mathcal V_r(\M) \ge \mu r^m$ for
sufficiently small $r$.  Due to the compactness of $\M$, this further gives an
upper bound on $\mathcal N_{r}(\M)$ (and thus for $\mathcal N_r (\M^\alpha) \le
\mathcal N_r(\M)$).  Thus for fixed $\eps$ and $\alpha$, $P_n$ forms an
$\eps$-sample for $\M^\alpha$ with probability $1$ as $n \to +\infty$. 

We remark that Lemma 7.3 of \cite{CGOS13} also states that $\mathcal
N_{\eps/2}(\M^\alpha) < +\infty$ (i.e, it is finite) and $\mathcal
V_{\eps/2}(\M^\alpha) > 0$ for the more general case where $\M$ is a complete
Riemannian manifold with bounded absolute sectional curvature, for any $\eps <
2\rho_c(\M)$.  Hence again for fixed $\eps$ and $\alpha$, $P_n$ forms an
$\eps$-sample for $\M^\alpha$ with probability $1$ as $n \to +\infty$. 

Putting Theorem \ref{thm:epssample-treedistance} and \ref{thm:epssample}
together, we obtain the following: 

\begin{theorem}\label{thm:splittree}
Let $\M$ be a compact $m$-dimensional Riemannian manifold embedded in
$\mathbb{R}^d$ with positive strong convexity radius. Let $f: \M \to \mathbb R$
be a $\mylip$-Lipschitz probability density function supported on $\M$.  Let
$P_n$ be a set of $n$ points sampled i.i.d. according to $f$.  Assume that we
are given a density estimator such that $\| f - \tilde f_n\|_\infty$ converges
to $0$ as $n \to \infty$.  For any fixed $\eps > 0$, we have, with probability
$1$ as $n \to \infty$, that $d(\optT, \HmyT) \le (4\mylip+1) \eps$, where
the parameter $\myr$ in computing the \splittree{} $\myT$ is set to be $2\eps$. 
\end{theorem}

\begin{proof}
Set $\lambda$ in Lemma~\ref{thm:epssample-treedistance} to be $2\mylip \eps$.
We then have that $d(\optT, \HmyT) \le 2 \mylip \eps + 2\mylip \eps + 2 \|f -
\tilde f_n \|_\infty$ if $P_n$ is an $\eps$-sample of $\M^\lambda$.  Since $\| f
- \tilde f_n \|_\infty$ converges to $0$ as $n$ tends to $\infty$, there exists
$N_\eps$ such that $\|f - \tilde f_n\|_\infty \le \eps$ for any $n> N_\eps$.
Hence $d(\optT, \HmyT) \le (4\mylip + 1) \eps$ if $P_n$ is an $\eps$-sample of
$\M^\lambda$ and for $n > N_\eps$.  The theorem follows from this and Theorem
\ref{thm:epssample} above. 
\end{proof}


\end{document}